\documentclass[10pt]{article} 
\usepackage[legalpaper, margin=1.5in]{geometry}

\usepackage[american]{babel}
\usepackage[T1]{fontenc}
\usepackage[utf8]{inputenc}

\usepackage{authblk}
\usepackage{hyperref}

\RequirePackage{amsthm,amsmath}
\RequirePackage[numbers,sort]{natbib}

\usepackage{float,amssymb,amsthm}
\usepackage{bbm}

\usepackage{xcolor}
\newcommand{\PA}{\mathrm{pa}}
\newcommand{\CH}{\mathrm{ch}}

\newcommand{\AN}{\mathrm{an}}
\newcommand{\DE}{\mathrm{de}}

\newcommand{\ADJ}{\mathrm{adj}}
\newcommand{\MB}{\mathrm{mb}}

\newcommand\independent{\protect\mathpalette{\protect\independenT}{\perp}}
\def\independenT#1#2{\mathrel{\rlap{$#1#2$}\mkern2mu{#1#2}}}

\DeclareMathOperator{\var}{Var}

\DeclareMathOperator*{\argmin}{arg\,min}
\newcommand{\iid}[0]{i.i.d.\xspace}

\newcommand{\norm}[1]{\lVert{#1}\rVert}

\newcommand{\PP}[1]{\mathbb{P}\left\{{#1}\right\}} 

\def\R{\mathbb{R}}

\newcommand{\dna}{\not\rightsquigarrow}

\usepackage[ruled,vlined,linesnumbered]{algorithm2e}
\usepackage{listings}
\usepackage{pgf, tikz}
\usetikzlibrary{arrows, automata,calc}
\tikzstyle{c1}=[circle,thick,auto,draw,minimum size=0.5mm,fill=lightgray!10]

\usepackage{subcaption}
\theoremstyle{plain}
\newtheorem{theorem}{Theorem}
\newtheorem{lemma}{Lemma}

\newtheorem{definition}{Definition}
\newtheorem{corollary}{Corollary}
\theoremstyle{definition}

\let\emptyset\varnothing

\title{Definite Non-Ancestral Relations and Structure Learning}
\author[1]{\href{mailto:Wenyu Chen <wenyuc@uw.edu>?Subject=Your UAI 2021 paper}{Wenyu Chen}{}} 
\author[2]{Mathias Drton}
\author[3]{Ali Shojaie}
\affil[1]{%
    Department of Statistics\\
    University of Washington\\
    Seattle, WA, USA
}
\affil[2]{%
    Department of Mathematics\\
    Technical University of Munich\\
    M\"unchen, Germany
}
\affil[3]{
Department of Biostatistics\\
    University of Washington\\
    Seattle, WA, USA}

\begin{document}
	
\maketitle

\begin{abstract}
In causal graphical models based on directed acyclic graphs (DAGs), directed paths represent causal pathways between the corresponding variables.  The variable at the beginning of such a path is referred to as an ancestor of the variable at the end of the path.  Ancestral relations between variables play an important role in causal modeling.  
In existing literature on structure learning, these relations are usually deduced from 
learned structures and 
used for orienting edges \citep{Claassen2011} or formulating  constraints of the
space of possible DAGs \citep{Magliacane_Claassen_Mooij_2017}.
However, 
they are
usually not posed as immediate target
of inference.
In this work we investigate the graphical characterization of ancestral relations via CPDAGs and d-separation relations. 
We propose a framework that can learn 
definite non-ancestral relations without first learning the skeleton. This framework yields structural information that can be used in both score- and constraint-based algorithms to learn causal DAGs more efficiently. 
\end{abstract}

	\section{Introduction}
	Directed Acyclic Graphs (DAGs) are commonly used  as models  of causal relations between random variables in complex systems \cite{Pearl09,sprites2000}. 
	In this framework, every random variable is modeled to be a function of  other random variables (its causes) and  stochastic noise.  A DAG is used to represent the resulting model for the system of all variables.  The vertices of this DAG represent the random variables and edges represent  direct causal effects.
	The DAG aids, in particular, in understanding conditional independences that the model imposes on the joint distribution of the random variables.  These conditional independences provide an alternative characterization of the joint distribution and can be read off the graph using the concept of d-separation.  If a joint distribution satisfies all these imposed conditional independences, it is said to be Markov with respect to the DAG.
    Structure learning from observational data is then the problem of learning a DAG from  data sampled independently from a distribution that is Markov with respect to the DAG. 
    
    A crucial aspect of structure learning stems from the fact that the data-generating DAG
    may be non-identifiable:  	Many different DAGs may yield the same statistical
	model for the observational data at hand.  These DAGs form a Markov equivalence class, which is then the object to be learned.
	The Markov equivalence class can be uniquely recovered from conditional independences
	among the corresponding random variables, and it can be represented via a completed partially directed acyclic graph (CPDAG) \cite{sprites2000}.  
	Members of a Markov equivalence class share the same adjacencies  and 
	unshielded colliders \cite{Andersson1997},
	which are represented explicitly in the CPDAG. 
	However, there are also common structures that are implicit.

	In this work, we study \emph{definite ancestral} and \emph{definite non-ancestral} 
	relations, which are ancestral 
	and non-ancestral relations  
	shared by all members of a Markov equivalence class.
	We provide graphical interpretations of these relations in the CPDAG, 
	and we also provide a framework for reliably learning definite non-ancestral (DNA) relations without the need to recover the skeleton.  
	These relations not only provide causal interpretations (in the form of ``change in $X$ definitely does not cause change in $Y$''), 
	but also facilitate
	further structure learning.  Indeed, we show that learning DNA relations directly can greatly improve the statistical and computational efficiency of existing structure learning algorithms. 
	
	


	
	Existing structure learning methods can be broadly categorized into constraint-based, score-based, and hybrid approaches. 
    In constraint-based approaches, 
    a DAG is learned 
    via tests of 
    conditional independence. 
    In score-based approaches, 
    a score, for example BIC, is assigned to each DAG and then an algorithm
    searches for the DAG that 
    optimizes the score.  
    Hybrid approaches use schemes in which the two approaches inform each other. 
    In this paper, we will focus on two prominent examples of learning algorithms:  The PC algorithm \cite{sprites2000} and the Sparsest Permutation algorithm \citep{Raskutti2013,Solus2020}.
	The PC algorithm is the default constraint-based method and 
	hierarchically performs tests of 
	conditional independence  with conditioning sets of increasing size. 
	Under a \emph{faithfulness assumption}, the population version of PC algorithm outputs the correct CPDAG \cite{sprites2000}.  At the population level (i.e., with ``infinite data''), the faithfulness assumption merely requires that the conditional independences in the data-generating distributions coincide (to sufficient order) with d-separation relations in the DAG.  	However, good  performance
	of the sample version is only guaranteed when the assumption is strengthened to bound signals of conditional dependence away from zero, which is far more restrictive \citep{kalisch2007,Uhler2013}.  
    On the other hand, the Sparsest Permutation algorithm \citep{Raskutti2013,Solus2020}  
	is a hybrid learning method that 
	searches among all topological orderings.
	For each ordering, a DAG is inferred via conditional independence tests given the  ordering, and the number of edges in the DAG is used as the score. 
	The algorithm looks for the sparsest DAG under which the data-generating 
	distribution is Markov.
	The SP algorithm relies on weaker distributional 
	assumptions than PC, but comes at increased computational cost due to its score-based searching scheme. 
	
	As noted above, we propose here an algorithm to directly learn DNA relations.  This algorithm is constraint-based and derives the relations from conditional independences/dependences. 
	The learning algorithm is based on two well-known rules of conditional independence that has been studied extensively in \cite{entner2013,Claassen2011,claassen2013,Magliacane_Claassen_Mooij_2017}.
	We will then show that the learned DNA relations provide  \textit{order-constraining} information. 
	In other words, they define a subset of
	all possible topological orderings (or DAGs)
	that is guaranteed to contain one correct ordering (or DAG). 
	Therefore, DNA relations can be used to
	reduce the number
    of conditional independence tests needed for running the PC algorithm.
    Similarly, the ordering information provided by the DNA relations can also help significantly reduce the search space of the SP algorithm,
    which grows exponentially with graph size.
    Regarding this last point, we would like to highlight an independent preprint  \cite{Squires2020} that uses information from the moral graph to reduce the search space of SP.  Compared with that work, our framework of DNA relations is more general and can accommodate weaker assumptions.

	\section{Graphical preliminaries}
	Our exposition follows standard terminology in graphical modeling; see e.g.~Part I of \cite{gm_handbook}.  Here, we briefly summarize the most important concepts. 
	
		Let $G=(V,E)$ be a graph with vertex set $V$ and
	edge set $E$. 
	Then $G$ is directed (or undirected) if $E$ contains only
	directed (or only undirected) edges; these we denote simply as $u\to v$ (or $u-v$).  A mixed graph may contain both 
	directed and undirected edges.  Throughout, and even when considering mixed graphs, we will only consider graphs that are simple, i.e., there may be at most one edge between any pair of nodes.  Two nodes in a graph are \textit{adjacent} if they are linked by an edge.  The \textit{skeleton} of a (possibly mixed) graph $G$ is the undirected graph that has an edge between any pair of nodes that are adjacent in $G$.  A set of nodes is a \textit{clique} if any two nodes in the set are adjacent.

	A \textit{path} in a graph is a
	sequence of distinct and adjacent vertices.
	A \textit{directed path} is a path along directed edges following the
	arrowheads.  Adding an additional directed edge back to the first node
	gives a \textit{directed cycle}.  
	A \textit{directed acyclic
		graph (DAG)} is a directed graph without
	directed cycles.
	If a graph $G$ contains
	the edge $u\to v$, then $u$ is a \textit{parent} of its \textit{child}
	$v$.  
	We write $\PA(G,v)$ for the set of all parents of a node $v$.
	Similarly, $\CH(G,v)$ is the set of children of $v$.  
	If there is  a
	directed path $u \rightarrow \cdots \rightarrow v$, then $u$ is an
	\textit{ancestor} of its \textit{descendant} $v$. 
	The sets of ancestors
	and descendants of $v$ are denoted $\AN(G,v)$ and $\DE(G,v)$, respectively. By convention,
	$v\in\AN(G,v)$ and $v\in\DE(G,v)$. 
	A set of nodes $C$ is ancestral
	 if $\AN(G,v)\subseteq C$ for all $v\in C$. When clear, we may drop the indication of the graph $G$, writing, e.g., $\AN(v)$ only.

	A triple of vertices $(u,v,w)$ is \textit{unshielded} if $v$ is
	adjacent to both $u$ and $w$, but $u$ and $w$ are not adjacent.  A
	non-endpoint vertex $v$ on a path $\pi$ is a \textit{collider} on the
	path if both the edges preceding and succeeding it have an arrowhead
	at $v$. A non-endpoint vertex $v$ on a path $\pi$ which is not a
	collider is a \textit{non-collider} on the path.  An unshielded triple
	$(u,v,w)$ is called a \textit{v-structure} if $v$ is a collider on the
	path $(u,v,w)$.  The \textit{Markov blanket} of a vertex, 
	denoted $\MB(u)$, is the 
	union of vertices connecting to it through either an edge, or 
	a v-structure.  For a choice of two vertices $u$ and $v$, let $C\subset V\setminus\{u,v\}$.  If there is a path from vertex $u$ to vertex $v$ such that all of its colliders are ancestors of an element of $C$ and all its non-colliders are outside $C$, then $u$ and $v$ are \textit{d-connected} given $C$.  If this is not the case, then $u$ and $v$ are \textit{d-separated} given $C$. 
	
	\section{Definite Non-Ancestral Relations}
	A DAG $G$ is usually not uniquely identifiable from the distribution of observational data.  Instead, there may be other graphs $G'$ that entail the exact same d-separation relations. Together, these graphs form the \textit{Markov equivalence class}, $[G]$. Importantly,
 the graphs in  $[G]$ share their adjacencies and possibly also 
	some of their edge marks.  In other words, the equivalent graphs $G'\in[G]$ may differ only through reversals of directed edges and some of the directed edges may in fact be oriented the same way
	in all members of $[G]$; these latter set of edges can be unambiguously interpreted as direct causal effects. 
	The characterization and discovery of such  edges is well-studied \citep{Andersson1997,Meek1995}.
	
		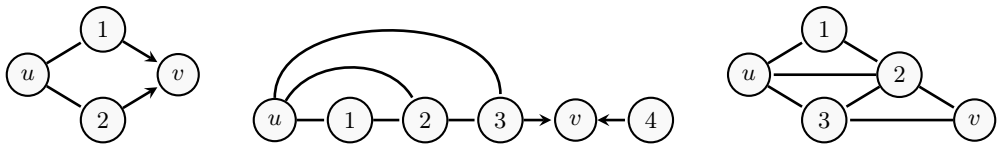
\begin{figure*}
		\centering
		\begin{tikzpicture}[> = stealth,shorten > = 1pt, auto,node distance = 1cm, semithick]
		
		\tikzstyle{every state}=[
		draw = black,
		thick,
		fill = white,
		minimum size = 2mm
		]
		
		\node[c1] (1) at(0,0) {\small$u$};
		\node[c1] (2) at(1,0.6) {\small1};
		\node[c1] (3) at(1,-0.6) {\small2};
		\node[c1] (4) at(2,0) {\small$v$};
		
		\draw[-,line width= 1] (1) -- (2);
		\draw[->,line width= 1] (2) -- (4);
		\draw[-,line width= 1] (1) -- (3);
		\draw[->,line width= 1] (3) -- (4);
		\end{tikzpicture}\qquad
		\begin{tikzpicture}[> = stealth,shorten > = 1pt, auto,node distance = 1cm, semithick]
		
		\tikzstyle{every state}=[
		draw = black,
		thick,
		fill = white,
		minimum size = 2mm
		]
		
		\node[c1] (1) at(0,0) {\small$u$};
		\node[c1] (2) at(1,0) {\small1};
		\node[c1] (3) at(2,0) {\small2};
		\node[c1] (4) at(3,0) {\small3};
		\node[c1] (5) at(4,0) {\small$v$};
		\node[c1] (6) at(5,0) {\small4};
		
		\draw[-,line width= 1] (1) -- (2);
		\draw[-,line width= 1] (2) -- (3);
		\draw[-,line width= 1] (3) -- (4);
		\draw[->,line width= 1] (4) -- (5);
		\draw[<-,line width= 1] (5) -- (6);
		\draw[-,line width= 1] (1) to [out=60,in=120] (3);
		\draw[-,line width= 1] (1) to [out=90,in=90] (4);
		\end{tikzpicture}\qquad
				\begin{tikzpicture}[> = stealth,shorten > = 1pt, auto,node distance = 1cm, semithick]
		
		\tikzstyle{every state}=[
		draw = black,
		thick,
		fill = white,
		minimum size = 2mm
		]
		
		\node[c1] (1) at(0,0) {\small$u$};
		\node[c1] (2) at(1,0.6) {\small1};
		\node[c1] (3) at(2,0) {\small2};
		\node[c1] (4) at(1,-0.6) {\small3};
		\node[c1] (5) at(3,-0.6) {\small$v$};
		
		\draw[-,line width= 1] (1) -- (2);
		\draw[-,line width= 1] (2) -- (3);
		\draw[-,line width= 1] (3) -- (4);
		\draw[-,line width= 1] (3) -- (1);
		\draw[-,line width= 1] (3) -- (5);
		\draw[-,line width= 1] (1) -- (4);
		\draw[-,line width= 1] (5) -- (4);
		\end{tikzpicture}
		\caption{In the graph on the left, 
		$A_{uv}=\{1,2\}$ is not a clique, so $u$ must have an arrowhead to one of them, and thus $u\rightsquigarrow v$ as suggested by Lemma~\ref{lem:cpdagda}.
		In the middle graph, $A_{uv}=\{3\}$, 
		which is trivially a clique, $u$ is not
		definitely ancestral to $v$. In 
		the last graph, $A_{uv}=\{2,3\}$,
		which forms a clique, and 
		it is possible to orient both 
		$2$ and $3$ into $u$ and make
		$u\notin\AN(v)$, and therefore $u$ is not definite ancestral to $v$ according to Lemma~\ref{lem:cpdagda}.}	\label{fig:da}
	\end{figure*}
	Throughout this paper, we will use the adjective `definite' to highlight the structure common to all members of a Markov equivalence class, $[G]$.  With this terminology, the edges that can be oriented the same way in all members of $[G]$ are
	definite (as opposed to incidental) arrows. 
	We similarly define definite ancestral  and definite non-ancestral  relations with regard to causal pathways. 

	\begin{definition}[Definite Ancestral  and Definite Non-Ancestral Relations]
			Let $u,v$ be two nodes in a DAG $G$. 
			 Then $u$ is definite ancestral to $v$, denoted $u\rightsquigarrow v$, if $u\in \AN(G',v)$ for all $G'\in [G]$. 
			 Similarly, $u$ is definite non-ancestral to $v$, or $u\dna v$, if $u\not\in \AN(G',v)$ for all $G'\in [G]$.  
			 When writing $u\rightsquigarrow W$  for a set of nodes $W$, we mean $u\in\AN(G',W)$  for all $G'\in [G]$; when writing $u\dna W$, we mean $u\dna w$  for all $w\in W$.   
			 Finally,  we define 
			 $D^{\rightsquigarrow}(G)=\{(u,v):u\rightsquigarrow v \text{ in }G\}$ and
			 $D^{\dna}(G)=\{(u,v):u\dna v \text{ in }G\}$.
	\end{definition}
	We emphasize that the two notions are not complementary to each other.  While the acyclicity of the graph entails that
	$u\rightsquigarrow v$ implies $ v\dna u$, the converse need not be true.  
	
	In the rest of this section, we present two  
	perspectives of definite ancestral and non-ancestral relations: the CPDAG perspective and the d-separation perspective. 
	The former provides a set of rules to read the (non-)ancestral relations off a CPDAG, and the latter enables efficient learning 
	without knowing the CPDAG.  This second perspective provides a foundation for learning DNA directly from a probability distribution, or data drawn from it.

	\subsection{Ancestral and non-ancestral relations from CPDAGs}

    Let $G$ be any DAG.  The Markov equivalence class $[G]$ can be represented using the completed partially directed acyclic graph (CPDAG) $G^*$.
	 The CPDAG is a mixed graph containing directed and undirected edges 
	 that has the same skeleton as $G$ and whose edges
	 are directed if and only if they have the same orientation in all members of the Markov equivalence class. 
	 The CPDAG representation is complete in the sense that all 
	 directed edges are definite arrows, and for each undirected edge $u-v$, there exists two DAGs in $[G]$ that
	 contain
	 $u\to v$ and $u\leftarrow v$,
	 respectively. 
	 
	 From the CPDAG perspective, definite ancestral and non-ancestral relations can be identified via 
	 \textit{possibly directed paths}, which are paths between 
	 two nodes with no arrow into the initial node. 
	 We say a possibly directed path 
	 is \textit{unshielded} if no three consecutive nodes on the path form a triangle. 
	 Notably, 
	 if there exists a possibly directed 
	 path from $u$ to $v$ in the CPDAG,
	 then some subsequence of this path forms an unshielded
	 possibly directed path \citep{zhang2008, Perkovic2018}. 
	 We now give a comprehensive characterization of definite ancestral and definite non-ancestral relations in terms of the CPDAG; all proofs are given in the supplement.
	 \begin{lemma}\label{lem:cpdagda}
		Let $u$ and $v$ be two nodes in  a CPDAG $G^*$. 
		\begin{itemize}
            \item Let $A_{uv}$ be the collection 
            of all nodes that lie immediately after $u$ on some unshielded possibly directed path to $v$.
			Then $u\rightsquigarrow v$ if and only if either $u$ has a definite arrow into $A_{uv}$ or $A_{uv}$ is 
            not a clique.
			\item  $u\dna v$ if and only if there is no possibly
			directed path from $u$ to $v$ in $G^*$.
		\end{itemize}
	 \end{lemma}
    Figure~\ref{fig:da} shows 
    examples of definite ancestral relations.
 	
 	\subsection{Ancestral and non-ancestral relations from d-separations}
 	Definite ancestral relations are
    easy to interpret, but somewhat delicate to read off the CPDAG. In contrast, definite non-ancestral 
    relations have more subtle interpretations
    but are easier to read off the CPDAG; they are also simpler to learn from $d$-separation relations.
    Next we present 
    well known rules about ancestral effects. 
    (See, e.g., \citep{entner2013,Claassen2011,claassen2013,Magliacane_Claassen_Mooij_2017}.)
	 \begin{lemma}[DA/DNA via d-separation]\label{lem:dna}
	 	Let $G=(V,E)$ be a DAG.  
	 	Let $u,v,x$ be distinct vertices, and let $W\subseteq V\setminus \{u,v,x\}$.
	 	Then the following holds:
	 	\begin{itemize}
	 		\item If $u,v$ are  d-connected given $W$ but d-separated given $W\cup x$, then $x\rightsquigarrow W\cup\{u,v\}$.

	 		\item If $u,v$ are  d-separated given $W$ but d-connected given $W\cup x$, then $x\dna u$, $x\dna v$ and $x\dna W$.
	 		\item If $u,v$ are d-separated marginally (i.e., given $\emptyset$), then $u\dna v$ and $v\dna u$.
	 	\end{itemize} 
	 \end{lemma}
	 
    	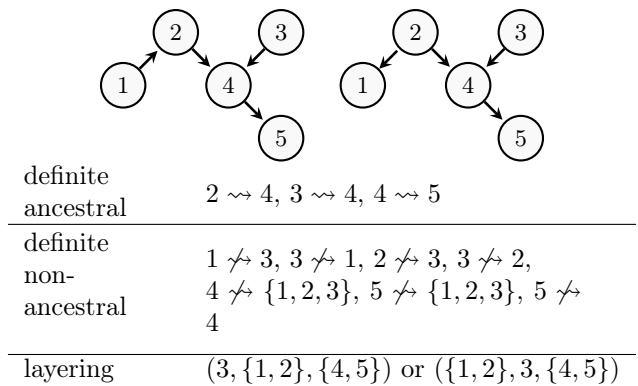
\begin{figure}
		\centering
			\begin{tikzpicture}[> = stealth,shorten > = 1pt, auto,node distance = 1cm, semithick,scale=1]
			
			\tikzstyle{every state}=[
			draw = black,
			thick,
			fill = white,
			minimum size = 2mm
			]
			
			\node[c1] (1) at(-0.7,0) {\small1};
			\node[c1] (2) at(0,0.7) {\small2};
			\node[c1] (4) at(0.7,0) {\small4};
			\node[c1] (3) at(1.4,0.7) {\small3};
			\node[c1] (5) at(1.4,-0.7) {\small5};

			\draw[->,line width= 1] (1) -- (2);
			\draw[->,line width= 1] (2) -- (4);
			\draw[->,line width= 1] (3) -- (4);
			\draw[->,line width= 1] (4) -- (5);
			\end{tikzpicture} \quad
			\begin{tikzpicture}[> = stealth,shorten > = 1pt, auto,node distance = 1cm, semithick,scale=1]
			
			\tikzstyle{every state}=[
			draw = black,
			thick,
			fill = white,
			minimum size = 2mm
			]
			
			\node[c1] (1) at(-0.7,0) {\small1};
			\node[c1] (2) at(0,0.7) {\small2};
			\node[c1] (4) at(0.7,0) {\small4};
			\node[c1] (3) at(1.4,0.7) {\small3};
			\node[c1] (5) at(1.4,-0.7) {\small5};
			
			\draw[<-,line width= 1] (1) -- (2);
			\draw[->,line width= 1] (2) -- (4);
			\draw[->,line width= 1] (3) -- (4);
			\draw[->,line width= 1] (4) -- (5);
			\end{tikzpicture}\vspace{0.01in}
			\begin{tabular}{ll}
					\parbox{2cm}{definite\\ ancestral\vspace{0.1cm}}  & 
						$2\rightsquigarrow 4$,
						$3\rightsquigarrow 4$,
						$4\rightsquigarrow 5$\\ \hline
					\parbox{2cm}{\vspace{0.1cm}definite\\ non-ancestral
					\vspace{0.1cm}}  &
					\parbox{5cm}{$1\not\rightsquigarrow 3$,
	$3\not\rightsquigarrow 1$,
	$2\not\rightsquigarrow 3$,
	$3\not\rightsquigarrow 2$,\\
	$4\not\rightsquigarrow \{1,2,3\}$,
	$5\not\rightsquigarrow \{1,2,3\}$,
						$5\dna 4$}\\ \hline
					layering &
							$(3, \{1,2\},\{4,5\})$ or $( \{1,2\},3,\{4,5\})$
			\end{tabular}
		\caption{A Markov equivalence class comprising two DAGs and its definite ancestral/non-ancestral relations. All DNA except for $5\dna 4$  can be discovered by Lemma~\ref{lem:dna}. }
		\label{tab:ex1}
	\end{figure}	
    
 	 The claims of Lemma~\ref{lem:dna} are illustrated in Figure~\ref{tab:ex1}. 
 	 The following result is
	 a special case of the second statement\footnote{This special case is discussed in \citet{Squires2020}. }:
	 \begin{corollary}\label{cor:dna}
	     If $u,v$ are  d-separated given $V\setminus \{u,v,x\}$ and are d-connected given $V\setminus \{u,v\}$, then  $x$ is a sink, i.e., $x\dna V\setminus \{x\}$.
	 \end{corollary}

    Note also that the characterization of definite (non-)ancestral relations  is not complete, meaning 
	there exist such relations that do not feature the configurations stated in Lemma~\ref{lem:dna}; see Figure~\ref{fig:dna}.

    In general, Lemma~\ref{lem:dna}  provides a way to identify DNA between two nodes, i.e., in the form of $u\dna v$. 
	However, we cannot identify definite ancestral relations between two nodes (see Figure~\ref{fig:da_counterex}).
	For this reason, 
	we only discuss
	learning 
	and applications of DNA
	in Section~\ref{sec:learning} and \ref{sec:applications}.
	
		\begin{figure}
		\centering
		\begin{tikzpicture}[> = stealth,shorten > = 1pt, auto,node distance = 1cm, semithick]
		
		\tikzstyle{every state}=[
		draw = black,
		thick,
		fill = white,
		minimum size = 2mm
		]
		
		\node[c1] (1) at(0,0) {\small 1};
		\node[c1] (2) at(1,0) {\small$v$};
		\node[c1] (3) at(0.2,0.7) {\small2};
		\node[c1] (4) at(2,0) {\small3};
		\node[c1] (5) at(3,0) {\small$u$};
		
		\draw[->,line width= 1] (1) -- (2);
		\draw[->,line width= 1] (3) -- (2);
		\draw[->,line width= 1] (2) -- (4);
		\draw[->,line width= 1] (4) -- (5);
		\end{tikzpicture}
		\caption{In this  example, $u\dna v$ but cannot be read from the DAG using  Lemma~\ref{lem:dna}.
		}
		\label{fig:dna}
	\end{figure}
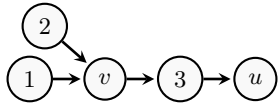
	
	\begin{figure}[t]
		\centering
		\begin{tikzpicture}[> = stealth,shorten > = 1pt, auto,node distance = 1cm, semithick]
		
		\tikzstyle{every state}=[
		draw = black,
		thick,
		fill = white,
		minimum size = 2mm
		]
		
		\node[c1] (1) at(0,0) {\small$v$};
		\node[c1] (2) at(0.8,0.6) {\small$w$};
		\node[c1] (3) at(1.6,0) {\small$v$};
		
		\draw[<-,line width= 1] (1) -- (2);
		\draw[->,line width= 1] (3) -- (2);
		\end{tikzpicture}\quad
		\begin{tikzpicture}[> = stealth,shorten > = 1pt, auto,node distance = 1cm, semithick]
		
		\tikzstyle{every state}=[
		draw = black,
		thick,
		fill = white,
		minimum size = 2mm
		]
		
		\node[c1] (1) at(0,0) {\small$v$};
		\node[c1] (2) at(0.8,0.6) {\small$w$};
		\node[c1] (3) at(1.6,0) {\small$v$};
		
		\draw[<-,line width= 1] (1) -- (2);
		\draw[<-,line width= 1] (3) -- (2);
		\end{tikzpicture}\quad
		\begin{tikzpicture}[> = stealth,shorten > = 1pt, auto,node distance = 1cm, semithick]
		
		\tikzstyle{every state}=[
		draw = black,
		thick,
		fill = white,
		minimum size = 2mm
		]
		
		\node[c1] (1) at(0,0) {\small$v$};
		\node[c1] (2) at(0.8,0.6) {\small$w$};
		\node[c1] (3) at(1.6,0) {\small$v$};
		
		\draw[->,line width= 1] (1) -- (2);
		\draw[<-,line width= 1] (3) -- (2);
		\end{tikzpicture}
		\caption{Three Markov equivalent DAGs, in all of which
		nodes $u$ and $v$ are marginally d-connected and d-separated given $w$; thus
		$w\rightsquigarrow \{u,v\}$. 
		However, 
		this does not imply 
		$w\rightsquigarrow u$ 
		or $w\rightsquigarrow v$. 
		}
		\label{fig:da_counterex}
	\end{figure}
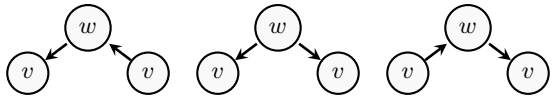

	\section{Learning DNA relations}\label{sec:learning}
    In this section, we discuss how to learn DNA relations from observational data. 
    Let $\Omega_0(\mathbb{P})=\{(u,v,S): u\independent v |S \textnormal{ in }\mathbb{P}\}$ be the collection of 
    all conditional independences, in the form of tuples, that hold in distribution $\mathbb{P}$. We can express the Markov property as follows.
	\begin{definition}[Markov Property]
		A distribution $\mathbb{P}$ is Markov with respect to  a DAG 
		$G$ if
		\[
		u,v\text{ are d-separated by }S \text{ in }G
		\implies (u,v,S)\in\Omega_0(\mathbb{P}).
		\]
	\end{definition}
	The reverse implication is known as the  
	faithfulness condition. 
	If $\mathbb{P}$ is Markov and 
	faithful to $G$, then the Markov equivalence class, $[G]$, can be recovered exactly from $\mathbb{P}$, or from observations obtained from $\mathbb{P}$. 

    Lemma~\ref{lem:dna}  states that  DNA relations can be identified from d-separation and a d-connection relations:
    DNA relations can be correctly learned from the  distribution if d-separation and d-connection relations correspond exactly to conditional independence and conditional dependence. 
    We formalize this condition as DNA-faithfulness. 
	\begin{definition}[DNA-faithfulness]\label{ass:dnafaith}
	    Let 
	    $\Omega$ and $\bar\Omega$ be two collection of triples consisting of two vertices and one set of nodes from a DAG $G$.  We say that a joint distribution $\mathbb{P}$ is 
	    DNA-faithful to  
	    $G$ with respect to  $(\Omega,\bar\Omega)$ if $\Omega\subseteq\Omega_0(\mathbb{P})$,  $\bar\Omega\cap\Omega_0(\mathbb{P})=\emptyset$, and it holds for any three nodes $u,v,z$
	    and set $S\subseteq V\setminus \{u,v,z\}$ that 
	    \begin{align}
	        (u,v,S)\in \Omega
	    \text{ and } (u,v,S\cup z)\in \bar\Omega \implies u,v \text{ are d-separated by } S \text{ in } G.
	    \end{align}
 	\end{definition}

 	We can learn DNA relations by a two-step procedure that first runs an algorithm to learn
	d-separation relations, and then performs additional tests to identify d-connections and learn DNA by Lemma~\ref{lem:dna}. Let $\mathcal{A}$  be an arbitrary constraint-based structure learning 	algorithm that tests and 
	collects a set of conditional independence statements $\Omega_\mathcal{A}(\mathbb{P})$ that hold in a distribution $\mathbb{P}$. 
 	Then we perform additional tests to detect conditional dependences in $\mathbb{P}$, i.e., we form
 	\[
	\bar\Omega_{\mathcal{A}}(\mathbb{P}) = 
	\{(u,v,S\cup z): u\not\independent v|S\cup z, (u,v,S)\in  \Omega_{\mathcal{A}}(\mathbb{P})   \}.
	\]
	We summarize this framework to learn DNA in Algorithm~\ref{alg:dna}.
	
    \begin{algorithm}[t]
		\caption{General DNA-learning framework}
		\label{alg:dna}
		\SetKwInOut{Input}{Input}
		\SetKwInOut{Output}{Output}
		\Input{An arbitrary constraint based algorithm $\mathcal{A}$}
		\Output{A set of DNA $D\subseteq D^{\dna}(G)$}
		$D\gets\emptyset$\;
		Run $\mathcal{A}$, record the conditional independences discovered by $\mathcal{A}$ as
		$\Omega_{\mathcal{A}}$\;
		\For{$(x,y,S)\in \Omega_{\mathcal{A}}$}{
		\For{$z\in V\setminus \{x,y,S\}$}{
		\If{$x\not\independent y|S\cup z$}{
		Record 
		$z\dna x$, $z\dna y$ and $z\dna S$ in $D$}
		}}
		\Return D.
	\end{algorithm}
	
	\begin{theorem}[Correctness of DNA Learning]\label{thm:dnalearning}
	    Let 
	    $\mathbb{P}$ be a distribution Markov to a DAG $G$, 
	    and let 
	    $\mathcal{A}$ be a constraint-based
	    learning algorithm
	    such that 
	    $\mathbb{P}$ is DNA-faithful to $G$ 
	    with respect to $(\Omega_\mathcal{A}(\mathbb{P}),\bar\Omega_\mathcal{A}(\mathbb{P}))$. Then, the output of Algorithm~\ref{alg:dna} is a set of true DNA relations in $G$. 
	\end{theorem}


	\subsection{Learning DNA from small d-sep sets}
	
	DNA learning begins by looking for d-separation relations.
	A classic approach for discovering d-separation relations systematically is to hierarchically test for conditional independences given sets of increasing size, $0,1,2,\ldots$.  In particular,  this is the idea behind the PC algorithm.  Adopting this strategy here, we can use the first few rounds of PC to learn d-separations. In other words, we use the PC algorithm with early stopping as the constraint-based procedure $\mathcal{A}$ in Algorithm~\ref{alg:dna}. The procedure is summarized in Algorithm~\ref{alg:dnasmall}.
	
	
	
	
	\begin{algorithm} 
		\caption{DNA-learning via small conditioning sets}
		\label{alg:dnasmall}
		\SetKwInOut{Input}{Input}
		\SetKwInOut{Output}{Output}
		\Input{A conditional independence test, a level $K$}
		\Output{A set of DNA relations $D\subseteq D^{\dna}(G)$}
		Repeat steps 2-5 of Algorithm~\ref{alg:dna} using PC with conditioning sets of size $0,\ldots,K$.\\
		\Return $D$.
	\end{algorithm}
	
	
	The rationale behind stopping the PC early when learning DNA is that the learning procedure $\mathcal{A}$ in Algorithm~\ref{alg:dna} only needs to find correct d-separations. However, false positive edges in the output of $\mathcal{A}$ are allowed. In fact, our empirical results in the supplement (Section~\ref{sec:numioscore})
	show that in most cases a large number of DNA relations can be learned from considering only the first two steps of the PC algorithm.

	

    With a view towards practical algorithms, we now focus on data from a multivariate Gaussian distribution $\mathbb{P}$ and adopt a threshold $\lambda$ to specify signal strengths.  Let $\rho(u,v|S)$ be the partial correlation obtained from the conditional distribution for the pair $(u,v)$ given $S$.  Define
    \begin{align*}
	\Omega_\lambda^{K} (\mathbb{P})&:=\{(u,v,S): |S|= K \text{ and } |\rho(u,v|S)|\leq \lambda\},\\
	\bar\Omega_\lambda^{K} (\mathbb{P})&:=\{(u,v,S): |S|= K \text{ and } |\rho(u,v|S)|> \lambda\}.
	\end{align*}
	Furthermore, let
	\[
	\Omega_\lambda^{\uparrow K} (\mathbb{P}):= \bigcup_{k=0}^K \Omega_\lambda^{k} (\mathbb{P})
	\]
    be the triples obtained 
    from conditioning sets of size at most $K$ and  strength threshold $\lambda$. 
    The next result guarantees the correctness of  Algorithm~\ref{alg:dnasmall}.

    \begin{theorem}[Correctness of DNA learning from small conditioning sets]
        \label{thm:forward}
		Let $G$ be a DAG, and let $\mathbb{P}$ be a distribution 
		that is Markov to $G$ and  
		DNA-faithful to $G$ with respect to 
		$\left(\Omega_\lambda^{\uparrow K}(\mathbb{P}),\bar\Omega_\lambda^{ K+1}(\mathbb{P})\right)$. 
		Then Algorithm~\ref{alg:dnasmall} 
		with  level $K$ returns 
		a correct set of DNA relations. 
	\end{theorem}
	
	For all choices of $K$, the DNA-faithfulness assumption with respect to $(\Omega_\lambda^{\uparrow K}(\mathbb{P}),\bar\Omega_\lambda^{ K+1}(\mathbb{P}))$
 	is weaker than the $\lambda$-strong faithfulness assumption for the PC algorithm; see \cite{Uhler2013}.
 	However, it cannot be directly compared with 
 	the adjacency and orientation faithfulness, which are the condition for completeness of PC 
    \citep{Ramsey2006}. 
    Nevertheless, in the special case of $K=0$,
    the DNA-faithfulness condition is  mild:
	If $\mathbb{P}$ is Gaussian, then it is sufficient if the correlation of every pair of marginally d-connected  nodes is bounded away from zero. 
	Our numerical result in Section~\ref{sec:numerical} also confirms that DNA-faithfulness is generally not stronger than other notions of faithfulness. 

	Given Theorem~\ref{thm:forward}, we next provide sample guarantees for linear structural equation models (SEMs) with sub-Gaussian errors.  In other words, we assume that the considered joint distribution $\mathbb{P}$ is that of a random vector $X$ that satisfies
	\[
	X = B^T X+ \varepsilon,
	\]
	where $B=(\beta_{uv})$ has entry $\beta_{uv}\not=0$ only if $u\to v$ is an edge in the considered DAG.  The vector $\varepsilon$ is comprised of independent random variables. In the following theorem we assume the distribution of the random vectors to be sub-Gaussian. 
	Given data we may form sample partial correlation $\hat\rho(u,v|S)$ and implement Algorithm~\ref{alg:dnasmall} by rejecting a hypothesis of conditional independence when $\hat\rho(u,v|S)>\lambda$.
	The following result covers both low- and high- dimensional problems. 
	\begin{theorem}[Sample guarantee for DNA learning from small conditioning sets]\label{thm:sample}
		Suppose data is generated as $n$ independent draws from the joint distribution $\mathbb{P}$ of a random vector $X\in \R^p$ that follows a linear SEM.  Let $\Sigma$ be the covariance matrix and assume that each $X_i/\Sigma_{ii}^{1/2}$
		is 
		sub-Gaussian with parameter $\sigma$. 
		Assume $\mathbb{P}$ 
		is Markov with respect to a DAG $G$, and also DNA-faithful
		to $G$ with respect to  $(\Omega_\lambda^{\uparrow K}(\mathbb{P}),\bar\Omega_\lambda^{K+1}(\mathbb{P}))$. 
		Assume 
		the minimal eigenvalues of all $(K+3)\times (K+3)$ submatrices of $\Sigma$
		are bounded below by $M>0$. For any 
		$\zeta>0$, 
		if the sample size satisfies
		\begin{align}
		    n\geq \left\{\log (p^2+p)-\log (\zeta/2)\right\}128(1+4\sigma^2)^2  
		    \max_i(\Sigma_{ii})^2(K+3)^2\left(\frac{M+1+2/\lambda}{M^2}\right)^2,
		\end{align}\label{eq:samplecomplex}
		then the output
		of Algorithm~\ref{alg:dnasmall} 
		with level $K$ and tests based on sample partial correlations is correct with probability at least 
		$1-\zeta$.
	\end{theorem}

    \subsection{Learning DNA from large d-sep sets}
    The d-separation relations for DNA can also be learned  tractably by testing conditional independences given large conditioning sets of size $p-2,p-3,\ldots$, where $p$ denotes again the number of considered variables. 
    In fact, by  Corollary~\ref{cor:dna}, 
	if $x,y$ are  d-separated given $V\setminus \{x,y,u\}$ and are d-connected given $V\setminus \{x,y\}$, then $u\dna V\setminus u$. 
	This means that we can learn DNA relations from the moral graph, which encodes the d-separation relations that hold between each pair of nodes when conditioning on all other nodes. A special case of this approach was implemented as a recursive algorithm in \citep{Squires2020}. The theorem below establishes the correctness of the general strategy for learning DNA relations.  
	
	\begin{algorithm} 
		\caption{DNA-learning via large conditioning sets}
		\label{alg:dnalearningback}
		\SetKwInOut{Input}{Input}
		\SetKwInOut{Output}{Output}
		\Input{A conditional independence test, a level $K$}
		\Output{A set of DNA relations $D\subseteq D^{\dna}(G)$}
		Initialize $\widetilde  V=V$\;
		\Repeat{$|\widetilde V|= p-K$}{
			$\widetilde M\gets $  moral graph over $\widetilde V$\;
		\For{$u\in \widetilde  V$}{
		    $M_u\gets $ moral graph over $\widetilde  V\setminus u$\;
		    $\widetilde M_u\gets $ subgraph of $\widetilde M$ over $\widetilde V\setminus u$\;
		    \If{\textnormal{ $\widetilde  M_u$ contains more edges than 
		    $M_u$}}{Record $u\dna \widetilde  V\setminus u$\;
		    Update
		    $\widetilde V\gets 
		    \widetilde V \setminus \{u\}$; Break\;}
		}
		}
		\Return $D$.
	\end{algorithm}
	
	Given a joint distribution $\mathbb{P}$, define
	\[
	\Omega_\lambda^{\downarrow K}(\mathbb{P}) =\bigcup_{k=p-K}^p\Omega_\lambda^{k} (\mathbb{P}).
	\]
	
	\begin{theorem}[Correctness of DNA learning from large conditioning sets]\label{thm:popback}
		Let $G$ be a DAG, and let $\mathbb{P}$ be a distribution that is Markov to $G$
		and DNA-faithful to $G$ with respect to $\left(\Omega_\lambda^{\downarrow K-1}(\mathbb{P}),\bar\Omega_\lambda^{K}(\mathbb{P})\right)$. Then Algorithm~\ref{alg:dnalearningback}
		with level $K$ returns 
		a correct set of DNA relations. 
	\end{theorem}
	
	Next, we establish sample consistency of the algorithm for  linear SEMs with sub-Gaussian errors. 
	\begin{theorem}[Sample Guarantee of DNA Learning from large conditioning sets]\label{thm:sampleback}
	    Consider the setup of Theorem~\ref{thm:sample}, but assuming DNA-faithfulness with respect to $\left(\Omega_\lambda^{\downarrow K-1}(\mathbb{P}),\bar\Omega_\lambda^{K}(\mathbb{P})\right)$.
		Let  
		\[
		\lambda^*=
		\min_{A\subseteq [p],|A|\geq p-K, i,j\in A, [(\Sigma[A])^{-1}]_{ij}\neq 0} \left|[(\Sigma[A])^{-1}]_{ij}\right|,
		\]
		where $\Sigma[A]$ denotes the $A\times A$ sub-matrix of $\Sigma$.
		Denote the sample covariance matrix as $S_n$.  
		
		(i) \textbf{Low-dimensional case}. 
		Assume 
	the minimal eigenvalue of  $\Sigma$
		is bounded below by $M>0$. If
		the sample size satisfies
		\begin{align}
		n\geq \{\log (p^2+p)-\log (\zeta/2)\}128(1+4\sigma^2)^2 
		\max_i(\Sigma_{ii})^2p^2\left(\frac{M+2/\lambda^*}{M^2}\right)^2,
		\end{align}
		then the output of  Algorithm~\ref{alg:dnalearningback} at level $K$ 
		and with conditional independence tests that hard-threshold the inverses of submatrices of $S_n$ with respect to $\lambda^*/2$
		is correct with high probability.
		
		(ii) \textbf{High-dimensional case.} 
		Suppose there exists 
		a sequence of $\lambda_n$ satisfying $
		\lambda_n\leq 
			\lambda^* \norm{\Sigma^{-1}}_1^{-1}/4$
		such that $ 
		\lambda_n\geq \norm{\Sigma^{-1}}_1\norm{\Sigma - S_n}_\infty $ with high probability 
		as $n,p\to\infty$.
		Then the output of  Algorithm~\ref{alg:dnalearningback} at level $K$ 
		using CLIME \citep{Cai2011} with tuning parameter $\lambda_n$ for conditional independence tests
		is correct with high probability. 
	\end{theorem}

	\section{DNA Applications}\label{sec:applications}
		
	\subsection{Ordering constraints and Layering}
	A topological ordering of a DAG $G$, denoted $\pi$, is a total ordering of the vertices such that  $\pi(u)<\pi(v)$ for each edge $u\to v$ in $G$.
	Due to acyclicity, there always exists at least one such ordering,
	though the ordering might not be unique.  
    
    Without prior knowledge, learning DAGs from orderings requires checking all $p!$ possible orderings. However, 
    DNA relations constrain the set of possible orderings. More specifically, we will show in Lemma~\ref{lem:constraints} that if $u\dna v$, then there exists a valid topological ordering of $G$ with $\pi(u)>\pi(v)$. 
    
	In general, we say an ordering $\pi$ is compatible with a DNA set $D$ if $\pi(u)>\pi(v)$ for each $u\dna v$ in $D$.
	We say $D$ is an \textit{order-constraining DNA set} if $D$ contains no 
	DNA statement cycles, i.e., 
	we cannot follow a sequence of DNA statements in $D$, such as $u\dna v,v\dna w,\ldots$ and get back to $u$. 
	Given an arbitrary DNA set $D$, we can obtain an order-constraining subset by removing statements until there is no cycle. 
	Specifically, if $D$ is output of  Algorithm~\ref{alg:dna}, 
	then we only need to remove 
	one from each pair of 
	$(u\dna v, v\dna u)$.
	If $D$ is an order-constraining DNA
	set, then there must be some topological ordering that is compatible with both $D$ and some member of the Markov equivalence class.
	
	
	\begin{lemma}[Ordering constraints]\label{lem:constraints}
	   Let $G$ be a DAG. 
	   If $u\dna v$ in $G$, then 
	   there exists a topological ordering $\pi$ that is 
	   compatible with some  $G'\in [G]$ and  satisfies $\pi(u)>\pi(v)$.
	   If $D$ is an order-constraining DNA set for $G$,
	   then there exists a topological ordering $\pi$ 
	   that is compatible with some $G'\in [G]$ and satisfies $\pi(u)>\pi(v)$ for all 
	   $(u,v)\in D$.
	\end{lemma}
	
	Order-constraining DNA sets reduce the set of possible topological orderings to be considered in structure learning to a subset that is guaranteed to contain at least one correct ordering. 
	Therefore, this reduced set can be adopted in 
	score-based structure learning methods 
	such as Greedy Equivalence Search (GES) 
	and Sparsest Permutation (SP)
	to trim down their search space.

	Moreover, order-constraining DNA sets yield	\textit{layerings} of DAGs. 
	A \textit{layering} of a DAG is an ordered partition of the vertex set into layers, which must be such that there is no arrow pointing from one layer to a preceding layer \citep{Manzour2021}. 
	The finest layering we may hope to infer from conditional independence tests is the one given by the chain components of the CPDAG \cite{Andersson2006}.
	In addition to reducing the number of possible orderings, the layering of a DAG can also be used to develop efficient algorithms for learning DAGs. We present such an approach in the next section, but first show that layerings can be learned correctly from DNA.
    
	 	\begin{algorithm}
		\caption{Learning layering from DNA}
		\label{alg:realizeio}
		\SetKwInOut{Input}{Input}
		\SetKwInOut{Output}{Output}
		\Input{DNA set $D\subseteq D^{\dna}(G)$.}
		\Output{DAG layering $L$.}
		$\text{Sources}\gets\text{Sinks}\gets\emptyset$; $V'\gets V$\; 
		\Repeat{$V'=\emptyset$}{
        Find the smallest subset $S\subseteq V'$ such that $(u,s)\in D$ for all $u\in V'\setminus S$,  $s\in S$,
        or $(s,v)\in D$ for all $v\in  V'\setminus S$ ,$ s\in S$\;
        In the former case, 
        $\text{Sources}\gets [\text{Sources}, S]$; in the latter case, 
        $\text{Sinks}\gets [S, \text{Sinks}]$\;
        $V'\gets V'\setminus S$\;
		}
		\Return $L=[\text{Sources},\text{Sinks}]$.
	\end{algorithm}

	\begin{theorem}[DAG Layering via DNA]\label{thm:realizedio}
		Let $G$ be a DAG and  $D\subseteq D^{\dna}(G)$. 
		The output of Algorithm~\ref{alg:realizeio} is 
		a valid layering of $G$.
	\end{theorem}

	\subsection{Sparsest Permutation with DNA}\label{sec:DNASP}
	Let $G$ be a DAG and let
	$\pi$ be an arbitrary 
	ordering of its vertices. We define 
	$G^\pi=(V,E^\pi)$ as the DAG
	deduced via the following rule:
	\begin{multline}\label{eq:ordrule}
	(\pi(i),\pi(j))\in E^\pi \text{ iff } i<j \text{ and }\\\pi(i) \text{, }\pi(j)\text{ not d-separated  by} \{\pi(1),\ldots,\pi(j-1)\}\setminus \pi(i) \text{ in }G.
	\end{multline}
	Clearly, $G^\pi=G$ when 
	$\pi$ is a valid topological ordering of $G$. 
    We write $G^\pi(\Omega)$ if we replace the d-separations
	with conditional independences in 
	$\Omega$.

	The Sparsest Permutation (SP) algorithm is a hybrid structure learning method: it searches through the space of all orderings to minimize the edge count (which plays the role of a score) of DAGs induced by
	the orderings via the constraint-based rule \eqref{eq:ordrule}; that is, it finds
	\[
	\pi^* =\argmin_{\pi} |G^\pi(\Omega)|.
	\]
	Under the 
	\textit{Sparsest Markov Representation
	condition} (SMR) --- which requires that 
	for any $G'$ that is Markov to 
	$\mathbb{P}$, either $|G'|>|G|$ or $G'$ is Markov equivalent to $G$ ---
	SP recovers the correct Markov equivalence class of $G$ \citep{Raskutti2013, Solus2020}.
	Since SMR is a necessary condition for restricted-faithfulness \citep{Raskutti2013},
	the correctness of the SP algorithm relies on weaker assumptions than that of PC.

	SP can be implemented as a greedy algorithm because starting from any arbitrary ordering, there always 
	exists a non-increasing (in number of edges) sequence of 
	DAGs that ends up in the 
	correct Markov equivalence class \citep{Solus2020}.
	One particularly efficient greedy
	approach is the Triangle SP (TSP), 
	which moves from one ordering to the other by reverting \textit{covered arrows}, 
	that is, edges in the form of
	$u\to v$ satisfying $\PA(v)=\PA(v)\cup\{u\}$. 
	By repeatedly looking for covered arrow reversals that induce sparser DAGs, TSP will recover a DAG in the target Markov equivalence class \citep{Solus2020}. 
	 
	However,  greedy search can be  computationally burdensome.
	With an arbitrary initialization
	of the ordering,
	it may need to traverse a long  non-increasing sequence of DAGs to reach the target. 
	Moreover, since 
	all members of a Markov equivalence class are connected
	via non-increasing sequences, 	
	greedy search may spend
    steps exploring the same equivalent
	class before moving on to a sparser one.\footnote{To circumvent the problem of bad initialization and getting stuck in an equivalence class in practice, \citet{Solus2020} suggest implementing Greedy TSP with limited search depth $d$ and restarting with different initialization for $r$ times.
	} 
	
	From this perspective, incorporating DNA information can be very useful.
	On the one hand, an order-constraining DNA set reduces the search space and provides better initial orderings; 
	on the other hand, the layering information learned by Algorithm~\ref{alg:realizeio} breaks down the learning problem into smaller sized problems that are easier to tackle. 
	
	Given a valid layering, the true DAGs can be learned by applying SP to the first layer, then adjust all the lower layers by the first layer and repeat this process until the last layer. We call this recursive approach 
    the Layered-SP. 
    
    \begin{algorithm} 
		\caption{Layered-SP}
		\label{alg:DNASP}
		\SetKwInOut{Input}{Input}
		\SetKwInOut{Output}{Output}
		\Input{Constraint-based  algorithm $\mathcal{A}$}
		\Output{A topological ordering}
		$D\gets$ DNA relations learned by Algorithm~\ref{alg:dna} with $\mathcal{A}$\;
		$D'\gets$ an order-constraining subset of $D$ obtained by dropping cycle-inducing DNA relations\;
		$L=(L_1,\ldots, L_m)\gets$ layering deduced from $D$ by  Algorithm~\ref{alg:realizeio}\;
		Run SP on ${L_1}$ with initialization compatible with $D'$,
		and record the output ordering as $\pi_1$\;
		\For{$l=2,\ldots,m$}{
			Run SP on 
			${L_l}$ adjusted for $\cup_{i=1}^{l-1}{L_i}$
			with initialization  compatible with $D'$, 
			and 
			record the output ordering as $\pi_l$
		}
		$\pi\gets [\pi_1,\ldots, \pi_m]$\;
		\Return $\pi$.
	\end{algorithm}
	

    The next shows that our modified approach is correct.
	\begin{theorem}[Layered-SP]\label{thm:iosp}
        Let $G$ be a DAG.
        Suppose observational data 
        are drawn from  a distribution $\mathbb{P}$ that is
        Markov to $G$, 
        DNA-faithful 
        with respect to $(\Omega_\mathcal{A}(\mathbb{P}),\bar\Omega_\mathcal{A}(\mathbb{P}))$, 
        and SMR to $G$. 
        Then the output of  Algorithm~\ref{alg:DNASP}, denoted $\pi$, 
        satisfied 
        $G^\pi\in[G]$.
	\end{theorem}
	
	
    

	\subsection{PC algorithm with DNA}\label{sec:DNAPC}
    
    As noted before, the PC algorithm learns a CPDAG from conditional independencies given sets of  increasing sizes. 
    In the PC algorithm, we may leverage DNA information by  excluding non-ancestral neighbors in conditional independence tests. 
	When accessing the independence relation between $u$ and $v$ from the perspective of $u$ in a working skeleton $C$, instead of searching for d-separation sets among $\ADJ(C,u)$,
	we use the following rule:
	\begin{itemize}	    
	    \item If $u\dna v$, search  $
	\ADJ(C,u)\setminus \{w:w\dna u \} $; otherwise, search  
        $\ADJ(C,u)\setminus \{w:w\dna v \text{ and }w\dna u\}$.
	\end{itemize}
	
    In the next lemma we show the correctness of the modifications, and in Figure~\ref{fig:dnapcexamples} we present a concrete example.

    \begin{lemma}[DNA and d-sep]\label{lem:dnapc}
		The version of PC with neighborhood search replaced by the above rule is correct. 
	\end{lemma}

	It is worth noting that if $u\dna v$ and $v\dna u$, then $u$ and $v$ are non-adjacent in the true DAG. 
	However, the orientation step of PC requires the d-separator of $u,v$, and hence
	we cannot simply remove the edge $u-v$ without searching and testing.
	
	 	\begin{figure}
 		\centering
 		\begin{tikzpicture}[> = stealth,shorten > = 1pt, auto,node distance = 1cm, semithick]
 		
 		\tikzstyle{every state}=[
 		draw = black,
 		thick,
 		fill = white,
 		minimum size = 2mm
 		]
 		
 		\node[c1] (1) at(0,0) {\small1};
 		\node[c1] (2) at(4.5,0) {\small2};
 		\node[c1] (3) at(3.75,-0.6) {\small3};
 		\node[c1] (4) at(3,0) {\small4};
 		\node[c1] (5) at(2.25,-0.6) {\small5};
 		\node[c1] (6) at(1.5,0) {\small$u$};
 		
 		\draw[->,line width= 1] (1) -- (6);
 		\draw[->,line width= 1] (2) -- (3);
 		\draw[->,line width= 1] (2) -- (4);
 		\draw[->,line width= 1] (3) -- (4);
 		\draw[->,line width= 1] (4) -- (5);
 		\draw[->,line width= 1] (5) -- (6);
 		\draw[->,line width= 1] (4) -- (6);
 		\end{tikzpicture}
 		 		\caption{Algorithm~\ref{alg:dnasmall}
 		 		with level $K=0$ discovers  $u$ being a sink. When accessing conditional independencies among others, we can exclude $u$ from the search. 
 		 		}
 	\label{fig:dnapcexamples}
 \end{figure}
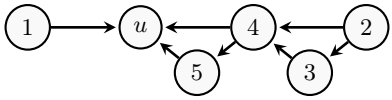

	\section{Numerical Results}\label{sec:numerical}
	    We examine the 
	performance of structure learning algorithms
	augmented by DNA. 
	We generate 200 
	random Erd\H{o}s-Renyi DAGs
	with $p=10$
	and expected neighborhood size 
	$s$ from 2 to 7. For each DAG, we
	build a linear SEM with coefficients
	drawn uniformly from $\pm 
	[0.3,1]$  and \iid Gaussian error with 
	variance drawn uniformly from $[1,2]$.

	We first compare the population versions of SP and PC with their DNA-modified versions. We plug in the conditional independence set 
	$\Omega_\lambda(\mathbb{P})=\{(i,j,S): |\rho(i,j|S)|\leq \lambda\}$ with $\lambda =\{10^{-2},10^{-3}\}$ where $\rho$ is the partial correlation. We report the number of conditional independence tests performed by each method, as well as their recovery rate, which is the proportion of times that the correct Markov equivalence class is recovered. 
	A higher recovery rate under fixed $\lambda$ means the method is less demanding with respect to faithfulness conditions, and is likely  more statistically efficient.  
	The results are shown in Figure~\ref{fig:recover}.
	The DNA version of both SP and PC have  higher recovery rate, showcasing the improvement on statistical efficiency 
	provided by our proposed algorithms. The number of conditional independence tests also highlight the computational gains by augmenting SP and PC with DNA. 
	Notably, even with low learning levels ($K=0$), the DNA modifications significantly reduce the total number of tests performed, especially when the graph is moderately sparse.
	Higher learning level ($K=1$) 
	provides more improvement, though it can increase the number of tests in some settings. 
	
	We also compare the sample versions of SP, PC and their DNA-modified versions. 
	To prevent false discoveries in the DNA Algorithm~\ref{alg:dnasmall}, we picked a large threshold for the d-connection step ($\lambda'=0.2$). 
	All other tests in SP and PC were performed
	at level $\lambda =0.02$.
	We draw 10000 samples from each SEM and use them to infer the CPDAG. We report the  
    recovery rate as well as the number of tests performed. 
	As expected, DNA provides improvement over both SP and PC when the underlying truth is moderately sparse, and the improvement is more significant for SP.  
	On the other hand, when the underlying true DAG is sparse, DNA could not provide much improvement, and false discoveries in DNA may hinder the performance. 
	
	\begin{figure}
	    \centering
	    \includegraphics[width=\linewidth]{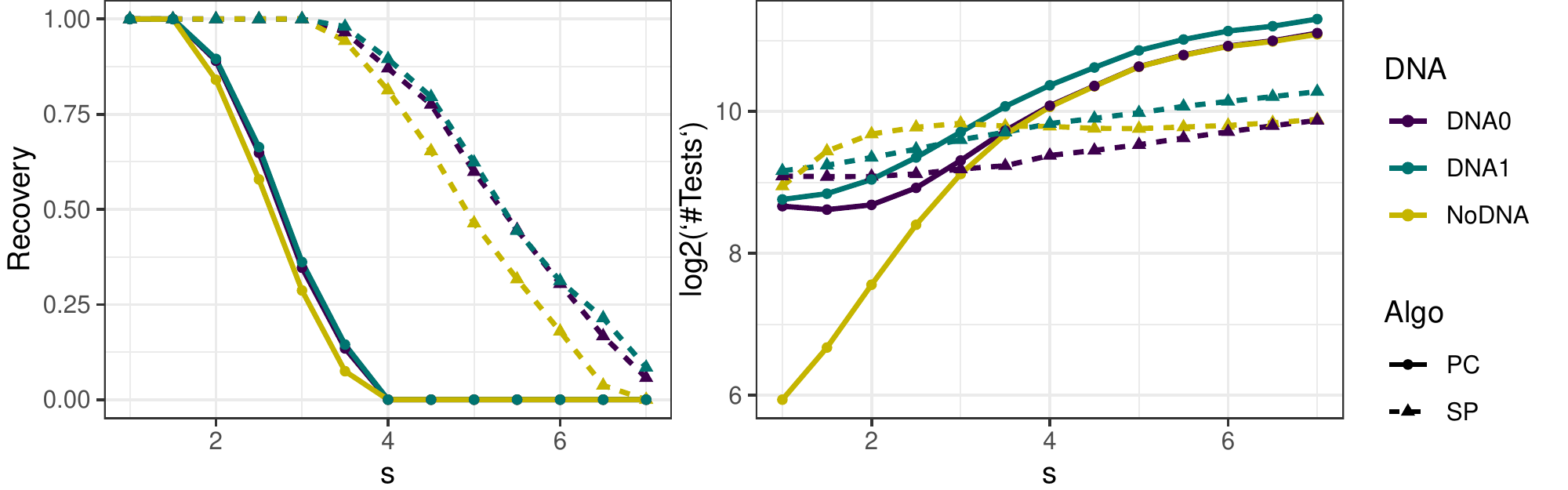}
	    \\
	    \includegraphics[width=\linewidth]{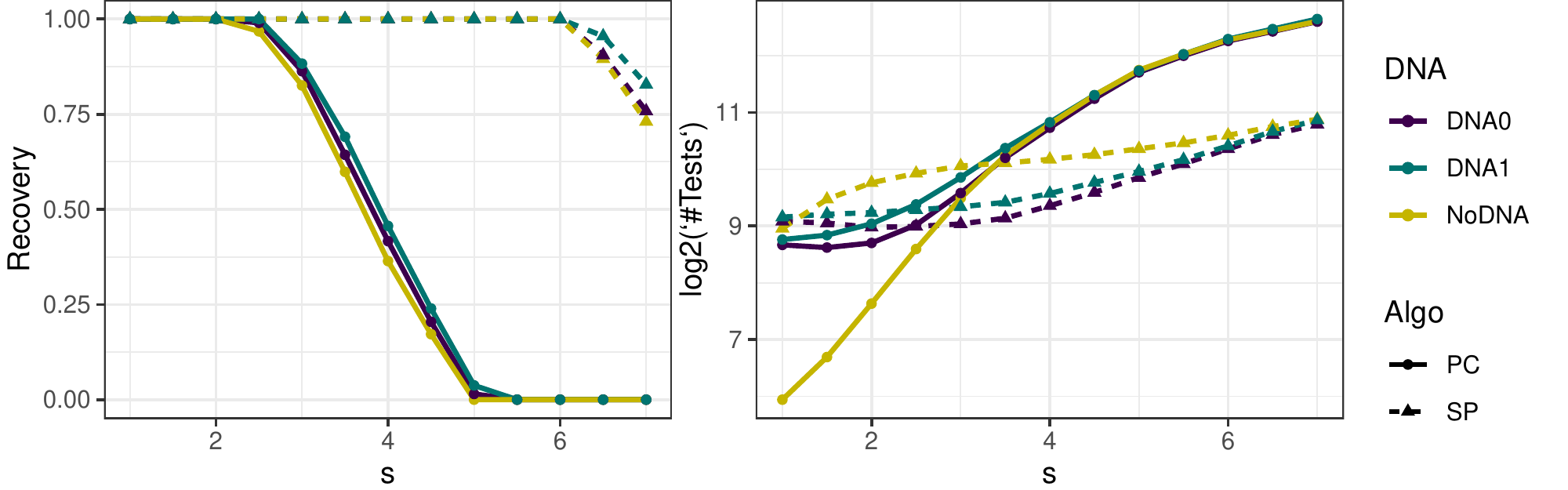}
	    \caption{Recovery rate (Left) and
	    Number of conditional independence tests (right) of the population version of SP, PC and their DNA modifications 
	    for random ER graphs with $p=10$ nodes, expected neighborhood size $s$ (x-axis), 
	    and $\lambda =0.01$ (top row) and $\lambda=0.001$ (bottom row).
	    }
	    \label{fig:recover}
	\end{figure}
	
	
	\begin{figure}
	    \centering
	    \includegraphics[width=\linewidth]{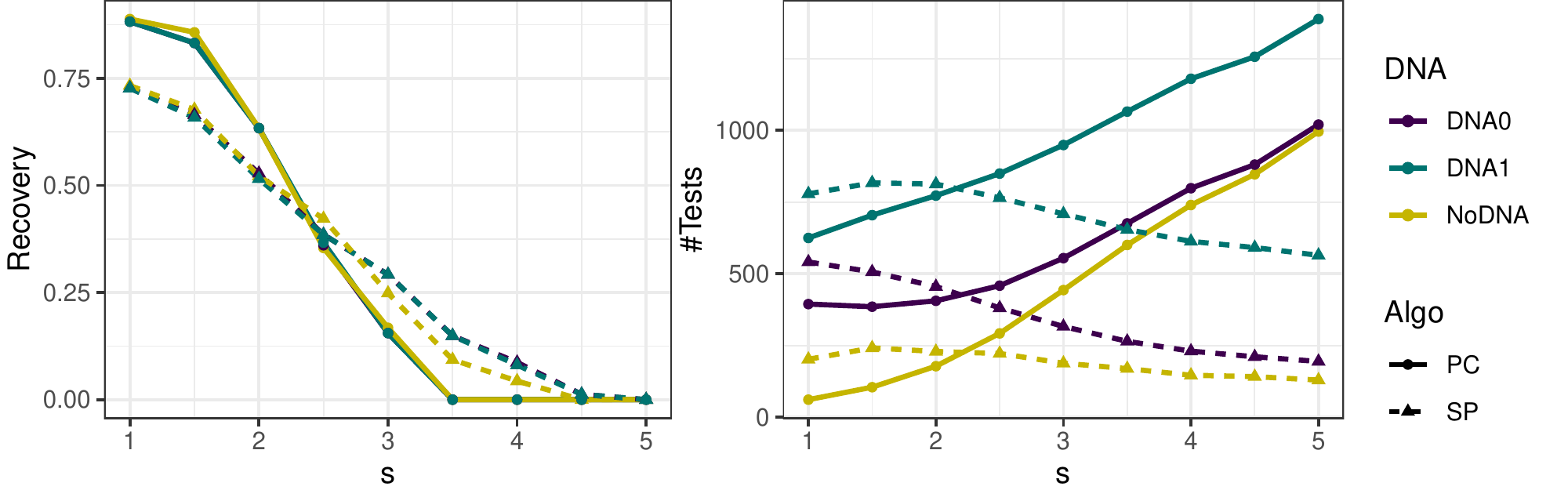}
	    \caption{Recovery rate (Left) and
	    Number of conditional independence tests (right) of the sample version of SP, PC and their DNA modifications 
	    performed on $n=10000$ samples from random ER graphs with $p=10$ nodes and
	     expected neighborhood size $s$ on x-axis. 
	    }
	    \label{fig:sample}
	\end{figure}

	\section{Conclusion}
	We introduced definite non-ancestral (DNA) relations as intermediate targets of inference in structure learning. 
	DNA relations can be learned from simple conditional independencies and lead to computational and statistical gains in DAG structure learning. DNA applications in graphs with latent variables would be interesting area of future research. 
	
	\section{Acknowledgements}
	This work has received funding from the U.S.\ National Institutes of Science (NSF) and of Health (NIH) under grants DMS-1161565, DMS-1561814, and R01GM114029, and from the European Research Council (ERC) under the European Union’s Horizon 2020 research and innovation programme (grant agreement No 883818)
	
	\bibliographystyle{apalike}
	\bibliography{bib}
	
	\clearpage
	\section{Appendix}
	
	\subsection{Proofs}
		\begin{proof}[Proof of Lemma~\ref{lem:cpdagda}]
 	    We first show the $\Rightarrow$ direction of the first statement. 
 		Suppose $u\rightsquigarrow v$
 		and $u$ does not have a directed arrow into $A_{uv}$. 
 		Suppose $A_{uv}$ is fully connected (or a singleton) and
 		define $B=\{a\to u: a\in A_{uv}\}$.
 		If the CPDAG has an extension that is consistent with $B$, then $u\notin\AN(v)$ in
 		this extension;  for a definition of extension, see e.g. \citep{dor1992}.  
 		We  apply the background knowledge  algorithm
        \citep{Perkovic2018},
        which is guaranteed to be  successful if such an extension exists. 
        Since $A_{uv}$ is fully connected,  at each step, 
        Meek's rules never orient  edges between nodes in $A_{uv}$; 
        therefore, also never
        orient any edge from $u$ into $A_{uv}$, 
        and the algorithm can enforce $B$
        without causing any conflict.
        But this means there exists a DAG in the Markov equivalence class (MEC) in which 
        $u\notin\AN(v)$. 
        Therefore we conclude the  $\Rightarrow$ direction of statement 1. 
 		The $\Leftarrow$
 		direction: 
 		since $A_{uv}$ is not a clique, 
 		there must be two nodes 
 		$a,a'\in A_{u,v}$ that are
 		non-adjacent and $(a,u,a')$ is not a v-structure. 
 		Then every DAG in the MEC
 		must have either
 		$u\to a$ or $u\to a'$, and
 		therefore a directed path to $v$.
 		
 		Now we show the second statement. 
 		The $\Rightarrow$ direction: if there exists a possibly 
 		directed path from $u$ to $v$, then 
 		there must be an unshielded possibly
 		directed path. However,
 		this means the 
 		 first edge on this path is 
 		oriented out of $u$ 
 		in some DAG in the MEC, in which case $u\in\AN(v)$. The $\Leftarrow$ direction follows directly from our definition. 
 	\end{proof}
 	
 		\begin{proof}[Proof of Lemma~\ref{lem:dna}]
		In the first statement,
		$x$ must block some path $\pi=(u,\ldots,s,x,t,\ldots,v)$ that is  d-connected given  $W$. Therefore $x$ is a non-collider on $\pi$, and the two subpaths $\pi_{xu}=(x,s\ldots,u)$ 
		and $\pi_{xv}=(x,t,\ldots,v)$
		are both d-connected 
		given $W$.  
		In each DAG of the MEC, 
		we can pick $\pi_{xu}$ or 
		$\pi_{xv}$, whichever starts with an outgoing arrow (one of them must do so, since $x$ is not a collider) and follow  non-collider arrows until we either reach $\{u,v\}$ or encounter a collider which is unblocked by $W$ (meaning it is ancestral to $W$). In the first case $x\in\AN(u\cup v)$ and in the second case $x\in\AN(W)$. 
		
		In the second statement, 
		$x$ must unblock some path
		$\pi=(u,\ldots,t,\ldots,v)$  that is blocked by $W$, 
		where $t$ is a collider on $\pi$
		and ancestor of $x$ (or $t=x$). Note that $t$ (and also $x$) is d-connected to both $u$ and $v$ given $W$. For this reason $x$ must not be ancestral 
		to $W$. 
		On the other hand, if $x$ is ancestral to $u$, then 
		there exists a directed path  $\pi'=(x,\ldots,u)$. 
		But then we obtain a d-connecting path between $u,v$ given $W$
		by gluing together $\pi'$, the directed path $(x,\ldots,t)$ and the subpath of $\pi$ from $t$ to $v$, which is a contradiction. 
	\end{proof}

		\begin{proof}[Proof of  Theorem~\ref{thm:dnalearning}, \ref{thm:forward} and \ref{thm:popback}]
		These three results are direct consequences of 
		Lemma~\ref{lem:dna} and
		DNA-faithfulness with respect to the corresponding conditional independence (CI) statements checked by the learning steps. 
	\end{proof}

	\begin{proof}[Proof of Theorem~\ref{thm:sample}]
		With the stated sample size, we may apply the error propagation computation in Lemma 1 of \cite{ravikumar2011} and in Lemma 6 of \cite{harris2013}, which gives the following bound:
		\begin{equation}
		\PP{\max_{i\neq j, |S|\leq K+1}\big| \rho(i,j|S)-\widehat\rho(i,j|S) \big|\geq \lambda/2}\leq
		\zeta.
		\end{equation}
		Consequently, with probability at least $1-\zeta$, the sets 
		$\Omega_{\lambda}^{\uparrow K}(\mathbb{P})$ and $\bar\Omega_{\lambda}^{ K+1}(\mathbb{P})$ are correctly inferred from the data, and therefore the DNA output is correct. 
	\end{proof}


	\begin{proof}[Proof of Theorem~\ref{thm:sampleback}]
	    The low-dimensional result 
	    can be obtained from an error propagation computation that is entirely analogous to the one from the proof of Theorem~\ref{thm:sample} above. In this case, we have 
	    $\PP{\|\Sigma^{-1}-S_n^{-1}\|_\infty>\lambda^*/2}\leq \zeta$ and consequently all moral graphs are correctly inferred from data. 
	    In the high-dimensional case, the theory for 	    the CLIME estimator $\widehat\Omega$ guarantees that $\norm{\widehat \Omega-\Sigma^{-1}}_\infty \leq 4\norm{\Sigma^{-1}}_1 \lambda_n$ for
	    $\lambda_n\geq \norm{\Sigma^{-1}}_1\norm{\Sigma - S_n}_\infty $ \cite{Cai2011}. The moral graphs of subgraphs are also correct \cite[see Lemma 5 of][]{Ghoshal2018}.
	    In both cases, the assumption on non-zero entries in the inverse covariance matrix guarantees that non-sinks will not be misspecified as sink. 
	    Therefore the algorithm is correct by Theorem~\ref{thm:popback}.
	\end{proof}
	
	\begin{proof}[Proof of Lemma~\ref{lem:constraints}]
	To show the first statement:
	Let $\pi_0$ be an arbitrary ordering of $G$ and suppose  $u\dna v$ is discordant with $\pi_0$. We claim we can swap the ordering to obtain a new ordering that is compatible with $G$ and
	$u\dna v$.
	We write $\pi_0=(X,u,Y,v,Z)$. 
	Denote $A=Y\cap \AN(v)$. Since $u\dna v$, we also have
	$u\dna A$. 
	Since $\pi_0$ is valid for $G$, there is no edge between $u$ and $A$, no edge between $Y\setminus A$ and $v$, and all edges between $A$ and $Y\setminus A$ are in the form of $A\to Y\setminus A$. So the new ordering $\pi'_0=(X,A,v,u,Y\setminus A,Z)$ is valid for $G$. 
	If $D$ is order-constraining, 
	then applying the swap operation above does not create new discordant pairs, and therefore an ordering can be swapped according to $D$ until it agrees with both $D$ and $[G]$. 
	\end{proof}
	
	\begin{proof}[Proof of Theorem~\ref{thm:realizedio}]
    In the output 
    $L=(L_1,\ldots,L_m)$, for each 
    $k=2,\ldots, m$, 
    it holds that 
    $v\dna \cup_{i=1}^{k-1} L_i$ for all 
    $v\in L_k$. Consequently, all edges between $L_k$ and layers preceding it must be directed into $L_k$ in $G$. Hence, $L$ satisfies the requirements to be a valid layering of $G$. 
	\end{proof}
	
	\begin{proof}[Proof of Theorem~\ref{thm:iosp}]
    Under the Markov and the DNA-faithfulness assumption, 
    $D$ is a DNA set of $G$ and, thus, by Theorem~\ref{thm:realizedio},
    $L$ is a  layering of $G$. It is now sufficient to show the DAG can be learned by recursively applying SP. 
	   We prove this claim by induction.
	   Under SMR, no graph on $L_1$ sparser than the subgraph of $G$ over $L_1$
       is Markov 
	   to the pattern of conditional independence (CI) relations among $L_1$. Therefore the output of SP on $L_1$ is consistent with the target MEC. 
	   Moving on to $L_1\cup L_2$, the sparsest graphs that are Markov to the CI relations among $L_1\cup L_2$ have $L_1$ ordered exactly as $\pi_1$. Hence, the optimal ordering of $L_2$ can be in an optimization that holds $\pi_1$ fixed and considers the joint distribution conditional on $X_{L_1}$. The general induction steps to the layers after $L_2$ proceed in the same way. 
	\end{proof}
	
	\begin{proof}[Proof of Lemma~\ref{lem:dnapc}]
	We show if  $u,v$ are d-separated by $S$, then they are also d-separated by $S\cap (\AN(u)\cup\AN(v))$.
	Let	$T=S\setminus (\AN(u)\cup\AN(v))$.
	The d-separation relation implies that every path between $u$ and $v$
	either has a non-collider in $S$ or a collider not in $S$ whose descendants are also not in $S$. 
	Consider an arbitrary path $\pi$ between $u$ and $v$.  If 
	$\pi$ does not go through $T$, then $\pi$ is blocked by $S\setminus T$.
	If $\pi$ contains some  $z\in T$ and $z$ is a collider on $\pi$, 
	then $\pi$ is blocked by $S\setminus T$ for not including $z$.
	If $z$ is not a collider on $\pi$, 
	then we can follow the arrows from $z$ on $\pi$  until we reach a collider, call it $x$, i.e., $(z,\ldots, x,\ldots)$ is a subpath of $\pi$. 
	Since $x\in \DE(z)\subset T$ and $x$ is a collider on $\pi$ that is not included in $S\setminus T$, the set $S\setminus T$ d-separates  $u,v$.
	\end{proof}

	\subsection{DNA Learning with low learning levels}\label{sec:numioscore}
	In this section we demonstrate 
	that a large proportion of DNA relations can be learned by  Algorithm~\ref{alg:dnasmall} with low 
	learning levels. 
	We generate 1000 random Erd\H{o}s-Renyi or power-law DAGs, 
	and then run Algorithm~\ref{alg:dnasmall} 
	 at learning level $K=0,1$
	 with the conditional independence oracle.  We then deduce a layering of $G$ using  Algorithm~\ref{alg:realizeio}. 
	We report 
	the proportion of all DNA learned and
	the proportion of edges in $G$ that lie
	between the learned 
	layers. 
	A high proportion of inter-layer edges means $L$ is informative about $G$.
	
	The results  are shown in  Figure~\ref{fig:ioscore}.
	It is evident that 
	a large proportion of DNA can be learned by 
	Algorithm~\ref{alg:dnasmall} with very early 
	stopping, 
	even if the underlying graphs are large or dense. For this reason we recommend running
	Algorithm~\ref{alg:dnasmall} with small $K=$ (0 or 1). 
	The results also show that 
	the corresponding layering
	discovered by Algorithm~\ref{alg:realizeio}
	are
	most informative when the true DAGs are not too 
	dense nor too sparse. 
	This is due to the characterization of  Lemma~\ref{alg:dna} which relies on 
	a conditional independence and a conditional 
	dependence: if too dense, there are not enough independences; and if too sparse, there are not enough dependences. 
	
	\begin{figure}[t]
	    \centering
	    \includegraphics[width=\linewidth]{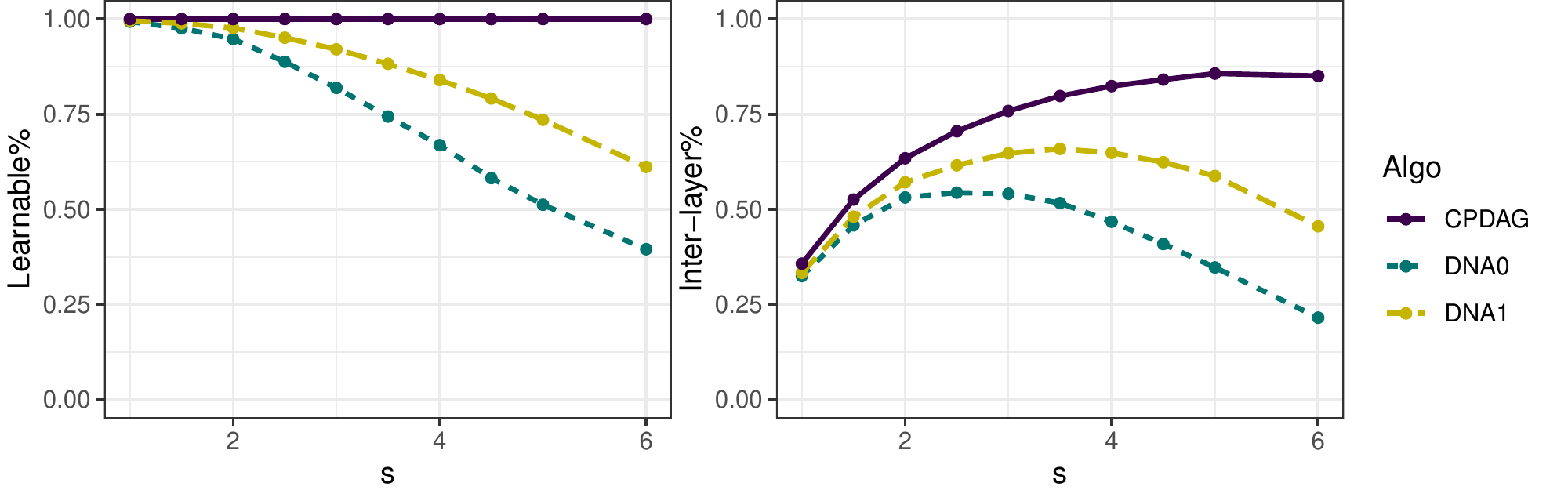}
	    \includegraphics[width=\linewidth]{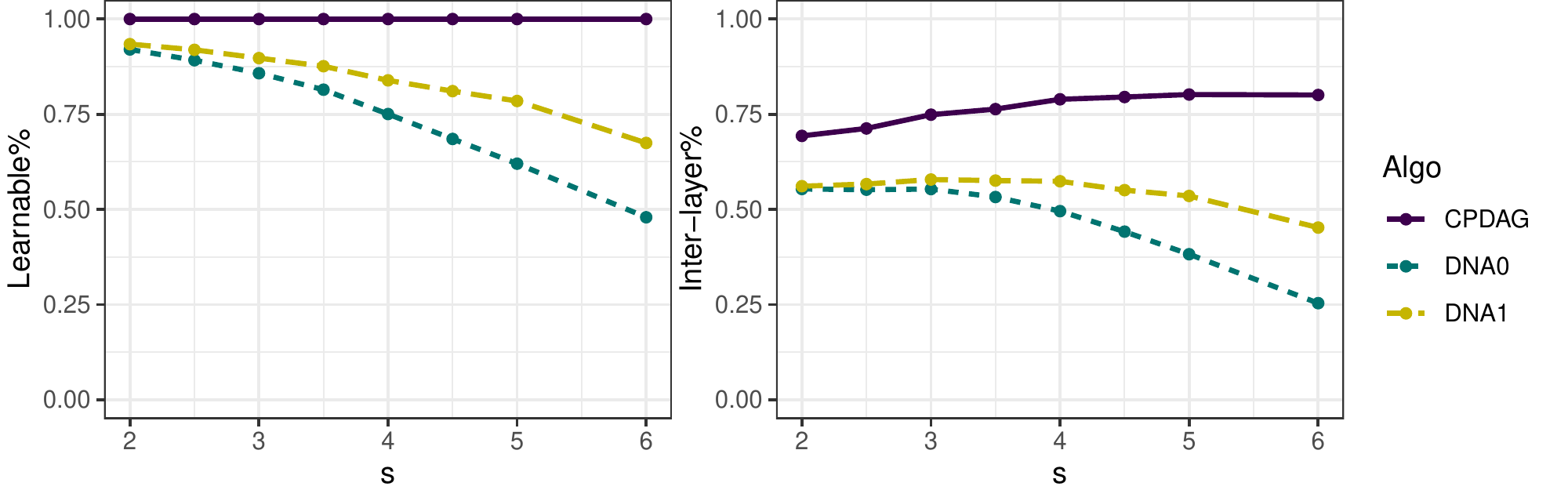}
	    \caption{
	    Proportion of DNA learned (left column),
	    and inter-layer edges (right column)
	    for Algorithm~\ref{alg:dnasmall} with levels 
	    0 and 1 in random Erd\H{os}-Renyi graphs (top row) and power-law graphs (bottom row) with $p=10$  vertices. The x-axis is average-node degree. }
	    \label{fig:ioscore}
	\end{figure}

\end{document}